\documentclass[conference,letterpaper]{IEEEtran}
\usepackage[T1]{fontenc}

\usepackage[letterpaper, left=0.69in, right=0.69in, bottom=1.05in, top=1.0in]{geometry}

\IEEEoverridecommandlockouts 

\usepackage{lipsum}  
\usepackage[utf8]{inputenc}
\usepackage{helvet}
\usepackage[T1]{fontenc}
\usepackage{url}
\usepackage{ifthen}
\usepackage{dsfont}
\usepackage{cite}
\usepackage[]{footmisc}
\usepackage{graphicx}
\usepackage[cmex10]{amsmath} 

\usepackage{amsmath,amssymb,amsthm}
\usepackage{xargs} 
\usepackage[pdftex,dvipsnames]{xcolor}  
\DeclareMathOperator{\E}{\mathbb{E}}
\usepackage[colorinlistoftodos,prependcaption,textsize=tiny]{todonotes}
\newcommandx{\unsure}[2][1=]{\todo[linecolor=red,backgroundcolor=red!25,bordercolor=red,#1]{#2}}
\newcommandx{\change}[2][1=]{\todo[linecolor=blue,backgroundcolor=blue!25,bordercolor=blue,#1]{#2}}
\newcommandx{\info}[2][1=]{\todo[linecolor=OliveGreen,backgroundcolor=OliveGreen!25,bordercolor=OliveGreen,#1]{#2}}
\newcommandx{\improvement}[2][1=]{\todo[linecolor=Plum,backgroundcolor=Plum!25,bordercolor=Plum,#1]{#2}}
\newcommandx{\thiswillnotshow}[2][1=]{\todo[disable,#1]{#2}}

\newtheorem{theorem}{Theorem}
\newtheorem{lemma}{Lemma}

\newtheorem{proposition}{Proposition}

\theoremstyle{definition}
\newtheorem{remark}{Remark}

\def \extended {1} 

\makeatletter
\def\thanks#1{\protected@xdef\@thanks{\@thanks
        \protect\footnotetext{#1}}}
\makeatother

\makeatletter
\renewcommand\footnoterule{\relax\kern-5pt
\hrule
\kern4.6pt}
\makeatother

\begin{document}
\title{Generalized Dual Discriminator GANs} 



\title{Generalized Dual Discriminator GANs}

\author{
    \IEEEauthorblockN{
        Penukonda Naga Chandana*,
        Tejas Srivastava*\thanks{* These authors contributed equally to this work.} \thanks{Gowtham Kurri acknowledges the support of ANRF, Govt. of India, under project  ANRF/ECRG/2024/005472/ENS.},
        Gowtham R. Kurri,
        V. Lalitha
        \\
    }
    \IEEEauthorblockA{International Institute of Information Technology, Hyderabad \\
    \\
    Email: \{penukonda.chandana@students., tejas.srivastava@students., gowtham.kurri@, lalitha.v@\}iiit.ac.in}

}

\maketitle


\begin{abstract}
   Dual discriminator generative adversarial networks (D2 GANs) were introduced to mitigate the problem of mode collapse in generative adversarial networks. In D2 GANs, two discriminators are employed alongside a generator: one discriminator rewards high scores for samples from the true data distribution, while the other favors samples from the generator. In this work, we first introduce {dual discriminator $\alpha$-GANs (D2 $\alpha$-GANs)}, which combines the strengths of dual discriminators with the flexibility of a tunable loss function, $\alpha$-loss. We further generalize this approach to arbitrary functions defined on positive reals, leading to a broader class of models we refer to as generalized dual discriminator generative adversarial networks. For each of these proposed models, we provide theoretical analysis and show that the associated min-max optimization reduces to the minimization of a linear combination of an $f$-divergence and a {reverse} $f$-divergence. This generalizes the known simplification for D2-GANs, where the objective reduces to a linear combination of the KL-divergence and the reverse KL-divergence. Finally, we perform experiments on 2D synthetic data and use multiple performance metrics to capture various advantages of our GANs.
   

\end{abstract}
\section{Introduction}
Generative adversarial networks (GANs)~\cite{NIPS2014_f033ed80} have significantly advanced the field of generative models by offering a powerful framework for learning complex data distributions and generating high-quality samples. The nature and learning behavior of the GAN depend heavily on the choice of the objective function used. The standard approach follows the principle of maximum likelihood estimation, minimizing the Kullback-Leibler (KL) divergence \( D_{\text{KL}}(p_{\text{data}} \| p_{\text{model}}) \), which encourages $p_\text{model}$ to cover multiple modes of $p_\text{data}$, but may generate undesirable samples. Alternatively, minimizing the reverse KL-divergence \( D_{\text{KL}}(p_{\text{model}} \| p_{\text{data}}) \) promotes mode-seeking behavior, causing $p_\text{model}$ to converge on covering specific modes of the data-distribution, while ignoring others. This is unwanted since the trained generator might not be able to generate specific types of samples and focus on producing similar samples with high probability. 

Methods like W-GANs \cite{arjovsky2017wasserstein} and Unrolled-GANs \cite{metz2016unrolled} have emerged to mitigate the issue of mode collapse in GANs. W-GANs achieve this by employing the Wasserstein distance as the loss function, providing a stable training signal and promoting diverse sample generation. Wasserstein distance works well in capturing distance between shapes of distributions in higher dimensional spaces like images. Unrolled-GANs, on the other hand, incorporate the influence of future discriminator updates into the generator's loss, encouraging it to produce varied and realistic outputs. Noting that the vanilla GANs~\cite{NIPS2014_f033ed80} are closely connected to the binary cross-entropy loss function, arbitrary class probability estimator (CPE) loss based GANs are introduced in \cite{welfert2024addressing}. Various other GANs have been studied in the literature with different value functions, e.g., $f$-GAN \cite{nowozin2016f}, IPM
based GANs \cite{arjovsky2017wasserstein},\cite{sriperumbudur2012empirical},\cite{liang2021well}, Cumulant GAN \cite{pantazis2022cumulant}, RenyiGAN \cite{bhatia2021least} and $\mathcal{L}_{\alpha}$-GAN\cite{veiner2024unifying}. Contrary to previous approaches, dual discriminator GANs (D2 GAN),\cite{nguyen2017dual} addresses the problem of mode collapse by incorporating two discriminators where one discriminator favors samples coming from the real distribution, while the other discriminator rewards high scores for samples generated from the generator. The optimization problem for D2 GANs includes both KL-divergence and reverse KL-divergence, which helps to capture multiple modes, thus eliminating the mode collapse problem. 

In this work, we first introduce {dual discriminator $\alpha$-GANs (D2 $\alpha$-GANs)}, which combines the strengths of dual discriminators with the flexibility of a tunable loss function, $\alpha$-loss~\cite{arimoto1971information,LiaoKSC19,SypherdDKDKS22}. We further generalize this approach to arbitrary functions defined on the positive reals, leading to a broader class of models we refer to as generalized dual discriminator GANs. This unified framework recovers both D2 GANs and D2 $\alpha$-GANs as special cases.
By leveraging the adaptability of $\alpha$-loss, our method enables a principled balance between the objectives of the dual discriminators, improving both the diversity and quality of generated samples while enhancing training stability.

Our main contributions are as follows.
\begin{itemize}
    \item We propose dual discriminator GANs based on the $\alpha$-loss that smoothly interpolates exponential loss ($\alpha=0.5$), binary cross-entropy loss ($\alpha=1$) and soft $0$-$1$ loss ($\alpha=\infty$). Our formulation recovers D2 GANs as a special case.
    \item Under the assumption that the generator and discriminators have sufficient parametric capacity, we derive closed-form expressions for the optimal discriminators (Lemma~\ref{lemma:Lemma1}) and show that the resulting min-max optimization reduces to minimizing a linear combination of an $f$-divergence and a {reverse} $f$-divergence (Theorem~\ref{thm:thm1}).
    \item We further extend our framework to define dual discriminator GANs based on arbitrary functions from positive reals to the reals, thereby generalizing both D2-GANs and D2 $\alpha$-GANs. In this setting, we show that the corresponding min-max optimization reduces to minimizing a linear combination of an $f$-divergence and a {reverse} $f$-divergence, where $f$ is determined by the choice of the underlying functions. (Theorem~\ref{thm:them2}).
    \item Finally, we empirically demonstrate the benefits of generalized dual discriminator networks, showing that they enable more stable learning --- characterized by faster convergence and fewer training epochs --- compared to both vanilla GANs and D2-GANs.

\end{itemize}

The paper is organized as follows. In Section~\ref{section:prelim}, we cover the preliminaries, including the vanilla GAN framework, arbitrary CPE loss based GANs, and dual discriminator GANs. In Sections~\ref{section:dualdism-alpha} and \ref{section:generalGANs}, we present D2 $\alpha$-GANs and generalized dual discriminator GANs, respectively, along with their corresponding theoretical analysis. Section~\ref{section:exp} discusses the experiments conducted and their results, and Section~\ref{section:concl} concludes the paper with a related discussion.


\section{Preliminaries}\label{section:prelim}
In this section, we review vanilla GANs~\cite{NIPS2014_f033ed80}, arbitrary CPE loss based GANs~\cite{welfert2024addressing}, and dual discriminator GANs~\cite{nguyen2017dual}.

\subsection{Vanilla Generative Adversarial Nets}
The most basic version of GANs, first introduced by Goodfellow \textit{et al.} \cite{NIPS2014_f033ed80}, involves 2 agents,  \textit{the generator ($G$) and the discriminator ($D$)}. Both are defined by independent deep-neural nets and play a zero-sum, adversarial game. While the generator takes input noise to produce fake samples which resemble the real distribution, the discriminator tries to learn to differentiate between the fake samples produced by the generator according to $P_g$ and those coming from the true distribution $P_d$. As both agents try to learn from the outputs of the other, the optimization reduces to a min-max game between $D$ and $G$: 
\begin{equation}
    \inf _G \sup _D \mathbb{E}_{X \sim P_d}[\log(D(X))]+\mathbb{E}_{X \sim P_g}[\log (1-D(X))].
    \label{vgans-value-function}
\end{equation}

Goodfellow \emph{et al.}~\cite{NIPS2014_f033ed80} showed that when the discriminator class is rich enough, the optimization problem in \eqref{vgans-value-function} simplifies to minimizing the Jensen-Shannon divergence between the real and the generated distributions.

\subsection{Arbitrary CPE Loss Based GANs}


    The vanilla GANs framework, which is closely connected to binary cross entropy loss, is extended to incorporate any arbitrary class probability estimator (CPE) loss functions in \cite{welfert2024addressing}. In particular, the value function in \eqref{vgans-value-function} is generalized as
    \begin{equation}
    V(D,G) = \mathbb{E}_{X \sim P_d}[-\ell(1,D(X))] + \mathbb{E}_{X \sim P_{g}}[-\ell(0,D(X))],
    \label{General-Loss-GANs}
\end{equation}
where $l(y,\hat{y})$, for $y\in\{0,1\}$, $\hat{y}\in[0,1]$, is an arbitrary CPE loss function. They show that the resulting min-max optimization problem simplifies to minimizing an associated $f$-divergence between the real and generated distributions,  
\begin{equation}
    \mathbf{D}_f(P_d \| P_{g}) = \int_X P_g(x) f\Big(\frac{P_d(x)}{P_g(x)}\Big) \enspace dx,
    \label{f-divergence formula}
\end{equation}
\begin{equation}
    \text{where} \quad
    f(u) = -\inf_{p \in [0,1]} \Big(\ell(1,1-p) + u \cdot \ell(1,p)\Big).
    \label{f-divergence definition}
\end{equation}

They have also considered a special CPE loss, $\alpha$-loss~\cite{arimoto1971information,LiaoKSC19,SypherdDKDKS22}, given by, 
\begin{align}
&\ell_\alpha(y,\hat{y})=\ell_\alpha(\hat{y})\cdot \mathds{1}\{y=1\} + \ell_\alpha(1-\hat{y})\cdot \mathds{1}\{y=0\}, \nonumber
\\
&\text{where} \quad \ell_\alpha(p) = \frac{\alpha}{\alpha-1}\Big(1-p^{\frac{\alpha-1}{\alpha}}\Big), p\in[0,1],
\label{alpha-loss}
\end{align}

 which is a tunable class of loss functions that generalizes well known losses like exponential loss ($\alpha=0.5$), binary cross-entropy loss ($\alpha=1$) and soft $0$-$1$ loss ($\alpha=\infty$). This results in minimizing the Arimoto divergence given by,
 \begin{equation}
     D_{f_\alpha}(P\|Q) = \frac{\alpha}{\alpha-1} \Big(\int_\mathcal{X} (p(x)^\alpha + q(x)^\alpha)^\frac{1}{\alpha} dx - 2^\frac{1}{\alpha}\Big),
     \label{arimoto divergence}
 \end{equation}
 which generalizes the Jensen-Shannon divergence $(\alpha=1)$, squared Hellinger distance $(\alpha=\frac{1}{2})$ and total variation distance $(\alpha=\infty)$.

\subsection{Dual discriminator GANs}
To address the issue of mode collapse in vanilla GANs, the dual discriminator GANs (D2 GANs) \cite{nguyen2017dual} framework was introduced, incorporating two discriminators with complementary roles. The output of both these discriminators are values in $\mathbb{R}^+$ instead of probabilities in $[0,1]$, unlike the above two GANs. For a given point $x$ in the data space, the first discriminator $D_1(x)$ outputs a high score if the data point is sampled from the true distribution, while the second discriminator $D_2(x)$ produces a high score if $x$ is generated by the generator. This setup simulates a three-player min-max game, where the resulting optimization problem simplifies to minimizing a linear combination of the forward and reverse KL-divergences.
\begin{align}
    &\inf_G \sup_{D_1, D_2} \Big(c_1 \times \mathbb{E}_{X \sim P_{\text {d}}}\left[\log D_1(X)\right]+\mathbb{E}_{X \sim P_g}\left[-D_1(X)\right] \nonumber \\
    & +\mathbb{E}_{X \sim P_{\text {d}}}\left[-D_2(X)\right]+
    c_2 \times \mathbb{E}_{X \sim P_{g}}\left[\log D_2(X)\right]\Big) \label{d2-gans-value-function} \\
    =&\inf_G\Big( c_1 \cdot (\text{log}c_1 - 1) + c_2 \cdot (\text{log}c_2 - 1) \nonumber \\ 
    &+ c_1 \cdot D_\text{KL}(P_d \| P_g) + c_2 \cdot D_\text{KL}(P_g \| P_d).\Big)
    \label{d2-gans-optimization-final}
\end{align}

The two directions of KL-divergence are not equivalent \cite[Figure 14]{goodfellow2016nips}. Minimization of forward KL-divergence leads to mean-seeking behavior, averaging out multiple modes and assigning high probability wherever real data is present. On the other hand, reverse KL-divergence results in mode-seeking behavior, preferring a low probability wherever data is absent, even if it results in the model ignoring some of the modes. The forward KL-divergence assures variety, but also leads to unusual samples lying between modes of the data generating distribution.  On the other hand, reverse divergence ensures that model does not produce undesirable data outside of the real distribution modes, but fails to capture variety in the data distribution as a result \cite{goodfellow2016nips}.

Ideally, we don't want our model to output unwanted data while ensuring variety in the data produced. This calls for the need of both, forward as well as reverse KL-divergence in our final optimization function, simultaneously mitigating mode collapse and improving the fidelity of generated samples. The use of two discriminators ensures this. Additionally, hyper-parameters to control the component of each divergence help us tune the GAN to our specific use-case. 
Since the output values of the two discriminators are positive and unbounded, keeping $c_1$ and $c_2$ small can also help stabilize learning by controlling the contribution of the values of $D_1$ and $D_2$ to the value function in \eqref{d2-gans-value-function} therefore regulating the corresponding terms from exploding quickly.


\section{Dual discriminator $\alpha$-GANs}\label{section:dualdism-alpha}
The value function \eqref{d2-gans-value-function} in dual discriminator GANs is closely connected to binary cross-entropy loss and soft $0$-$1$ loss. We propose dual discriminator GANs based on $\alpha$-loss~\cite{arimoto1971information,LiaoKSC19,SypherdDKDKS22} that smoothly interpolates exponential loss ($\alpha=0.5$), binary cross-entropy loss ($\alpha=1$) and soft $0$-$1$ loss ($\alpha=\infty$). Formally, in dual discriminator $\alpha$-GANs (D2 $\alpha$-GANs), $G$, $D_1$ and $D_2$ play a three-player game with the corresponding optimization problem:
\vspace{0.1cm}
\begin{align}
    \inf_G \sup_{D_1,D_2} V_\alpha(G,D_1,D_2),
    \label{D2-gans-optimization-problem}
\end{align}
where the value function is given as
\footnote{We have extended the $\alpha$-loss $\ell_\alpha(p)$ in \eqref{alpha-loss} to include any positive real inputs. In particular, we define $l_\alpha:\mathbb{R}^+\rightarrow \mathbb{R}$, as $l_\alpha(t)=\frac{\alpha}{\alpha-1}(1-t^{\frac{\alpha-1}{\alpha}})$.}

\begin{align}
   &V_\alpha(D_1,D_2,G)=  \nonumber \\
    &c_1 \E_{X \sim P_d}\left[-\ell_{\alpha_1} (D_1(X))\right] + \E_{X \sim P_g}\left[\ell_{\alpha_2}(D_1(X)) - 1 \right] \nonumber
    \\ 
    +&\E_{X \sim P_d}\left[\ell_{\alpha_2}(D_2(X))-1\right] +
     c_2 \E_{X \sim P_g}\left[-\ell_{\alpha_1}(D_2(X))\right]. 
    \label{d2-alpha-gans-value-function} 
\end{align}
    
It is easy to see that D2 GANs value function \eqref{d2-gans-optimization-final} can easily be recovered by taking $\alpha_1\rightarrow1$ and $\alpha_2\rightarrow\infty$ as 
\begin{enumerate}
    \item $\lim_{\alpha\rightarrow1}\ell_{\alpha}(t)=-\log(t)$  \text{and}
    \item $ \lim_{\alpha\rightarrow\infty}\ell_{\alpha}(t) = 1-t$.
\end{enumerate}



We now show that, given $G, D_1 $ and $D_2$ have enough parametric capacity, the generator learns the real distribution. We begin by considering the optimization problem in \eqref{D2-gans-optimization-problem} and finding optimal discriminators for a fixed generator, assuming that the discriminators' capacity is sufficiently large.

\begin{lemma}\label{lemma:Lemma1}
For a fixed generator $G$, under the condition $\alpha_2 > \alpha_1$, the discriminators optimizing the $\sup$ in \eqref{D2-gans-optimization-problem} are given by
\vspace{0.1cm}
\begin{equation}
    D_1^*(x) = \left( \frac{c_1 P_d(x)}{P_g(x)} \right)^{\frac{\alpha_1\alpha_2}{\alpha_2 - \alpha_1}} \ \text{and} \ D_2^*(x) = \left( \frac{c_2 P_g(x)}{P_d(x)} \right)^{\frac{\alpha_1\alpha_2}{\alpha_2 - \alpha_1}} 
    \label{optimal-discriminators}
\end{equation}
\label{lemma-optimal-discriminators}
for all $x \in \mathcal{X}$.
\end{lemma}

\begin{proof}[Proof Sketch] The function $t \rightarrow a\cdot (-\ell_{\alpha_1}(t)) + b\cdot \big( \ell_{\alpha_2}(t)-1\big)$ achieves its maxima at $\mathbb{R}^+$ at $t^*=(a/b)^{\frac{\alpha_1\alpha_2}{\alpha_2-\alpha_1}}$ provided that $\alpha_2>\alpha_1$. The condition $\alpha_2>\alpha_1$ is required for the negativity of the second derivative of the aforementioned function for $t^*$ to be a point of {maxima}. See\if \extended 1 Appendix A for detailed proof.\fi \if \extended 0\cite[Appendix A]{gend2gans} for detailed proof.\fi

\end{proof}


\begin{remark}
When $\alpha\rightarrow 1$ and $\alpha_2 \rightarrow \infty$, Lemma \ref{lemma-optimal-discriminators} recovers \cite[Proposition 1]{nguyen2017dual}.
\end{remark}

\begin{remark} 
When $P_d = P_g$, from \textit{Lemma 1}, we obtain the optimal discriminators as, for all $x \in \mathcal{X}$, 
\begin{align}
    D_1^*(x) = \left( {c_1} \right)^{\frac{\alpha_1\alpha_2}{\alpha_2 - \alpha_1}} \ \text{and} \ D_2^*(x) = \left( {c_2} \right)^{\frac{\alpha_1\alpha_2}{\alpha_2 - \alpha_1}}. 
\end{align}
\end{remark}

This implies that, when the generated distribution is equal to the real distribution, the discriminator cannot distinguish generated data from the real data, i.e., the discriminators output constant values independent of the input.  

\begin{theorem}\label{thm:thm1}
Under the condition $\alpha_2 > \alpha_1$ and given fixed optimal discriminators $D_1^*$, $D_2^*$ according to Lemma 1, the optimization problem in \eqref{D2-gans-optimization-problem} simplifies to the following
\vspace{0.1cm}
\begin{equation}
    \inf_G \Big(c_1\cdot \mathbf{D}_{f^{\alpha_1,\alpha_2}_{c_1}}(P_d \| P_g) + c_2 \cdot \mathbf{D}_{f^{\alpha_1,\alpha_2}_{c_2}}(P_g \| P_d)\Big),
    \label{d2-alpha-gans-final-optimization}
\end{equation}
\textit{where $f^{\alpha_1,\alpha_2}_c$ is given as}
\vspace{0.2cm}
\begin{align}
    f^{\alpha_1,\alpha_2}_c(u) =& -\frac{\alpha_1}{\alpha_1-1}\Big( u-c^{\frac{\alpha_1\alpha_2-\alpha_2}{\alpha_2-\alpha_1}}\cdot u^{\frac{\alpha_1\alpha_2-\alpha_1}{\alpha_2-\alpha_1}}\Big) \nonumber
    \\
    &-\frac{\alpha_2}{\alpha_2-1}\Big(c^{\frac{\alpha_1\alpha_2-\alpha_2}{\alpha_2-\alpha_1}}\cdot u^{\frac{\alpha_1\alpha_2 - \alpha_1}{\alpha_2-\alpha_1}}\Big).
    \label{d2-alpha-gans-fc}
\end{align}

\textit{The infimum in \eqref{d2-alpha-gans-final-optimization} is obtained at $G^*$ such that $P_g^{*}=P_d$ as }
\begin{align}
    V_\alpha(&D_1^*, D_2^*, G^*) = -\frac{\alpha_1}{\alpha_1-1}(c_1+c_2)\nonumber
    \\
    + &\Big(\frac{\alpha_1}{\alpha_1-1}-\frac{\alpha_2}{\alpha_2-1}\Big)\cdot\Big( c_1^{\frac{\alpha_1\alpha_2-\alpha_1}{\alpha_2-\alpha_1}}+c_2^{\frac{\alpha_1\alpha_2-\alpha_1}{\alpha_2-\alpha_1}} \Big).
    \label{d2-alpha-gans-minimum-value}
\end{align}

\end{theorem}

\begin{proof}[Proof Sketch] We first substitute the optimal discriminators in \eqref{optimal-discriminators} into the value function \eqref{d2-alpha-gans-value-function}. A careful simplification of the expression thus obtained gives \eqref{d2-alpha-gans-final-optimization} with the convex function $f_c^{\alpha_1,\alpha_2}$ in \eqref{d2-alpha-gans-fc}. The infimum in \eqref{d2-alpha-gans-final-optimization} is obtained as $c_1 \cdot f^{\alpha_1,\alpha_2}_{c_1}(1) + c_2 \cdot f^{\alpha_1,\alpha_2}_{c_2}(1)$ since the minimum value of $\mathbf{D}_f(P\|Q)$ is $f(1)$ at $P=Q$.  \if \extended 1 For detailed proof, see Appendix B. \fi \if \extended 0 For detailed proof, see\cite[Appendix B]{gend2gans}.\fi
\end{proof}

\begin{remark}
When $\alpha_1\rightarrow1$ and $\alpha_2\rightarrow\infty$, Theorem 1 recovers the \cite[Theorem 1]{nguyen2017dual}. Analogous to the linear combination of KL-divergence and reverse KL-divergence in \eqref{d2-gans-optimization-final} for dual discriminator GANs, Theorem 1 implies that the corresponding min-max optimization problem in D2 $\alpha$-GANs simplifies to the minimization of a linear combination of $f_c$-divergence with $c=c_1$ and {reverse} $f_c$-divergence with $c=c_2$.
 \if \extended 1 A detailed proof is given in Appendix C. \fi \if \extended 0 A detailed proof is given in \cite[Appendix C]{gend2gans}.\fi

\end{remark}

\section{Generalized dual discriminator GANs}\label{section:generalGANs}
In this section, we propose dual discriminator GANs based on arbitrary functions from positive reals to the real values, thereby generalizing both D2 GANs and D2 $\alpha$-GANs. In particular, $G$, $D_1$ and $D_2$ play a three-player game with the corresponding optimization problem: 
\begin{align}
    \inf_G \sup_{D_1,D_2} V_\ell(G,D_1,D_2),
    \label{D2-gengans-optimization-problem}
\end{align}
where
\begin{align}
    &V_\ell(G, D_1, D_2) = \nonumber \\
        &c_1 \cdot \mathbb{E}_{X\sim P_d}[-\ell_{1}(D_1(X))] + \mathbb{E}_{X \sim P_{g}}[\ell_{2}(D_1(X))-1] \nonumber
        \\& + \mathbb{E}_{X \sim P_d}[\ell_{2}(D_2(X))-1] + c_2 \cdot \mathbb{E}_{X \sim P_{g}} [-\ell_{1}(D_2(X))]
        \label{general-loss-value-function}
\end{align}
and $\ell_1:\mathbb{R}^+\rightarrow \mathbb{R}$ and $\ell_2:\mathbb{R}^+\rightarrow \mathbb{R}$. 

 Under sufficient generator and discriminator capacities, we show that the min-max optimization in \eqref{D2-gengans-optimization-problem} reduces to minimizing a linear combination of an $f$-divergence and a {reverse} $f$-divergence, where $f$ is determined by the choice of the underlying functions $\ell_1$ and $\ell_2$.

\begin{theorem}\label{thm:them2}
    For any given functions $\ell_1:\mathbb{R}^+\rightarrow \mathbb{R}$ and $\ell_2:\mathbb{R}^+\rightarrow \mathbb{R}$, the min-max optimization in \eqref{D2-gengans-optimization-problem} simplifies to 
    \begin{align}\label{eqn:thm2:opt}
        \inf_G \Big( &c_1  \cdot \mathbf{D}_{f_{c_1}}(P_d\|P_g) + c_2 \cdot \mathbf{D}_{f_{c_2}}(P_g\|P_d) \Big),
        \end{align}
        {where}
        \begin{align}\label{eqn:fc}
        f_{c}(u)=\sup_{t \in \mathbb{R}^+}\Big(-u\cdot \ell_1(t) + \frac{\ell_2(t)}{c} \Big).
    \end{align}
\end{theorem}

\begin{proof}[Proof Sketch] We decompose the inner optimization problem in \eqref{D2-gengans-optimization-problem} into separate optimization problems over $D_1$ and $D_2$, and further simplify them to obtain \eqref{eqn:thm2:opt}. The function $f_c$ in \eqref{eqn:fc} is convex, as it is defined as the supremum of affine functions.  \if \extended 1 A detailed proof is given in Appendix D. \fi \if \extended 0 A detailed proof is given in \cite[Appendix D]{gend2gans}.\fi

\end{proof}

\begin{remark}
    The generalized dual discriminator GANs optimization problem \eqref{D2-gengans-optimization-problem}  recovers the optimization problem for D2 $\alpha$-GANs \eqref{D2-gans-optimization-problem} when $\ell_1=\ell_{\alpha_1}$ and $\ell_2=\ell_{\alpha_2}$.
\end{remark}


\section{Experiments}\label{section:exp}
We perform experiments on a synthetic dataset to demonstrate the effectiveness of D2 $\alpha$-GANs in its versatility and solving issues of mode collapse while generalizing a number of loss functions. 

\subsection{Synthetic Dataset : Mixture of Gaussians}
In this experiment, we re-create the experimental setup in \cite{metz2016unrolled} to generate a synthetic dataset of mixture of bivariate gaussians arranged in a circle. The real data distribution contains 8 Gaussians with a covariance matrix  0.02\textit{I}. The centers of the Gaussians are arranged uniformly in a circle of radius 2. The gaussians are well separated with almost no overlap. Such a setting is good for evaluating performance of our models with multiple regions of very high and very low densities so that issues like mode-collapse  can be captured effectively. As in \cite{nguyen2017dual} and \cite{metz2016unrolled}, the generator uses two ReLU 128-size hidden layers and a linear output to map noise to 2D samples, while the discriminator employs one ReLU layer and a softplus output for classification. This lightweight architecture efficiently learns simple data-distribution without unnecessary complexity, as the low-dimensional data requires minimal modeling capacity. 

\begin{figure}[h]
  \includegraphics[width=0.5\textwidth]{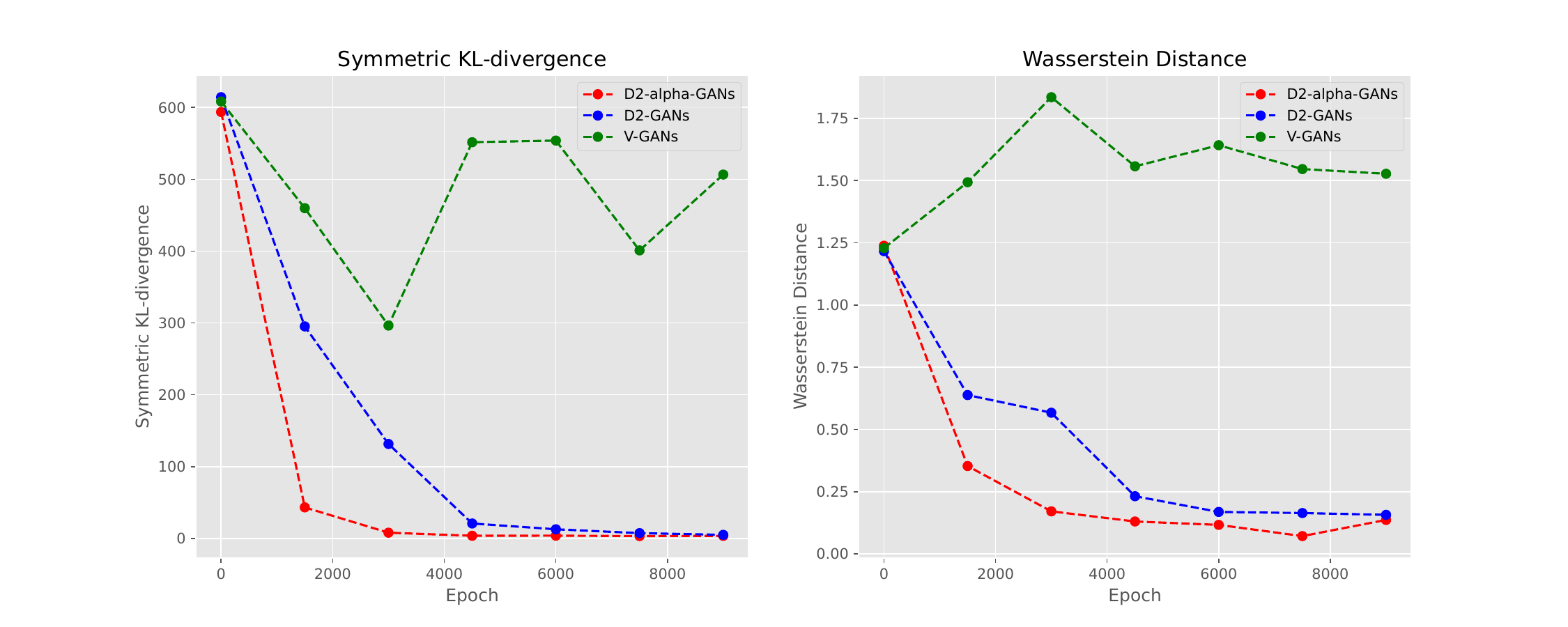}
  \caption{Comparison Symmetric KL-divergence and Wasserstein Distance for vanilla-GANs, D2 GANs, and D2 $\alpha$-GANs}
  \label{fig:MoG_distances}
\end{figure}
\begin{figure}[h]
  \includegraphics[width=0.5\textwidth]{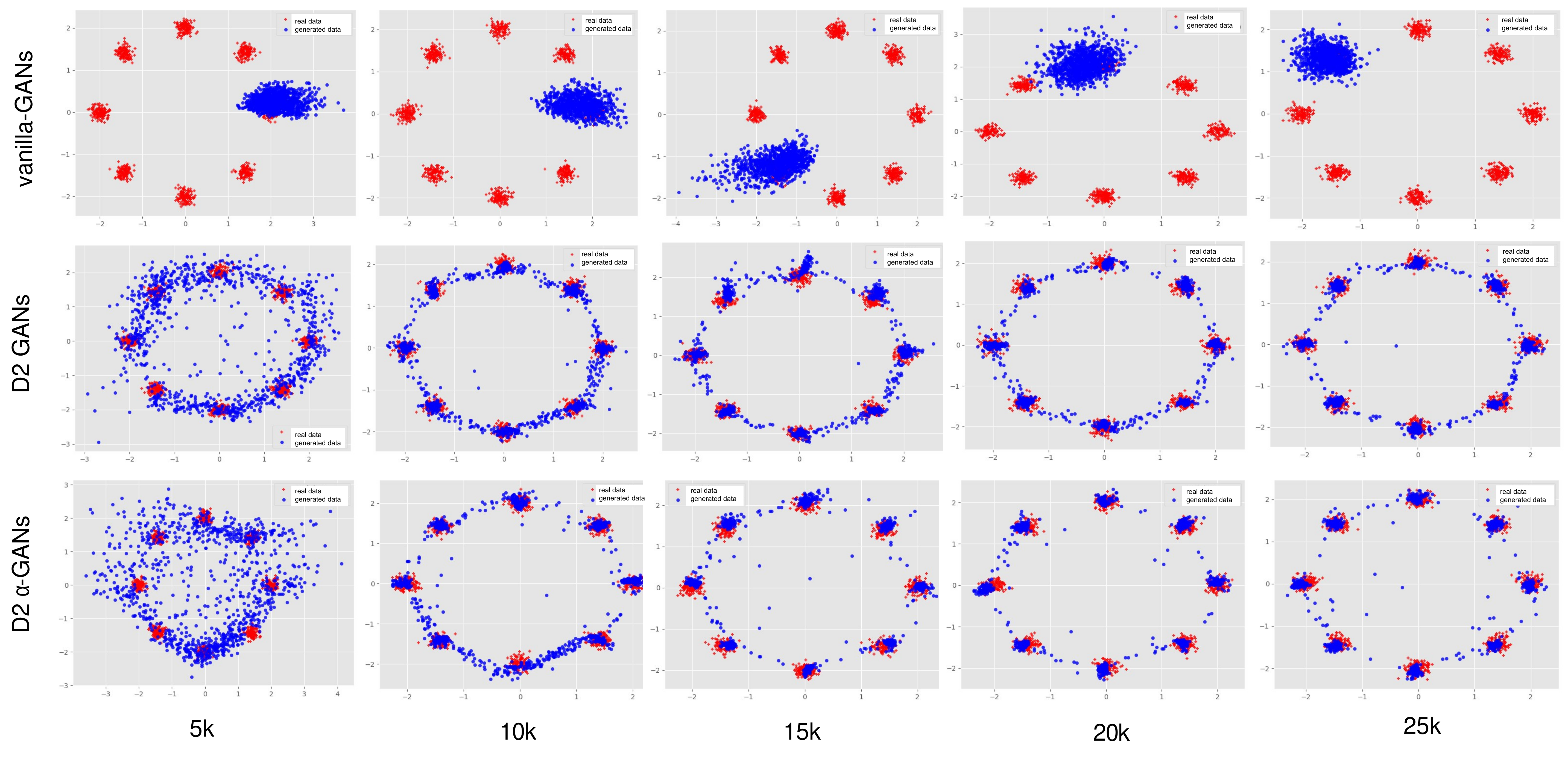}
  \caption{Visualization of Mode coverage for vanilla-GANs, D2 GANs and D2 $\alpha$-GANs for 25k epochs. Data sampled from the true distribution is in red and the data generated by the generator is in blue.}
  \label{fig:MoG_visualization}
\end{figure}

We train the model for 25k epochs with a batch-size of 512 and plot 1k samples every 5k epochs to track the evolution of our model. We have also performed hyperparameter tuning for our model with different values for $\alpha_1$, $\alpha_2$, $c_1$, $c_2$, learning rate and have considered one of the best set of parameters for demonstration. The values used for the analysis are $\alpha_1=0.6$, $\alpha_2=0.9$, $c_1=0.01$, $c_2=1.5$, learning rate = $0.001$ and random seed = $712$ (for D2 $\alpha$-GANs) and $c_1=1.2$, $c_2=1.0$ and learning rate = $0.0002$ (for D2 GANs). 
We can observe that while vanilla-GANs fails to capture multiple modes and collapses on different data-modes throughout the training, D2 GANs and D2 $\alpha$-GANs quickly spread around to capture all the modes of the data and later converge the distribution to capture all the modes separately. As compared to D2 GANs, D2 $\alpha$-GANs tend to capture all modes quickly and display a more stable learning. The same observation can be made from the plots of distance metrics like \textit{symmetric KL-divergence} and \textit{Wasserstein distance} as shown in Figure \ref{fig:MoG_distances}.

The symmetric KL divergence \cite{yao2025symmetric} is given by:
\begin{equation*}
    d(P_d,P_g) = D_{KL}(P_d\|P_g) + D_{KL}(P_g\|P_d),
\end{equation*}
and the Wasserstein distance \cite{arjovsky2017wasserstein} is given by,
\begin{equation*}
W({P}_d, {P}_g) = \inf_{\gamma \in \Pi({P}_d, {P}_g)} \mathbb{E}_{(x,y) \sim \gamma} \left[ \| x - y \| \right].
\end{equation*}
A lower value of symmetric KL divergence and Wasserstein distance signify a better model. While the Wasserstein metric, the distance from the model to the true distribution, approaches to zero for both D2 GANs and D2 $\alpha$-GANs, our model converges faster than D2 GANs, illustrating the superiority of our model.


\section{Discussion}\label{section:concl}
The use of $\alpha$-loss in place of the standard cross-entropy loss in the vanilla GANs value function introduces a tunable parameter, allowing $\alpha$-GANs~\cite{welfert2024addressing} to capture a range of GAN formulations within a unified framework. A natural question is whether it is possible to introduce two separate tunable parameters --- one for the discriminator and one for the generator --- within the same value function. However, the following proposition shows that such a formulation is not meaningful. The key reason is that when the data distribution $P_d$ matches the generator distribution $P_g$, the optimal discriminator must output random values, i.e., $D^*(x)=\frac{1}{2}$, for all $x\in\mathcal{X}$, which is not the case in this setting.

\begin{proposition}\label{prop1}
Consider the optimization problem:
\begin{align}
\sup_D \mathbb{E}_{X \sim P_d}[-\ell_{\alpha_1}(1,D(X))] + \mathbb{E}_{X \sim P_{g}}[-\ell_{\alpha_2}(0,D(X))].
\end{align}
When $P_d=P_g$, for the optimal discriminator to be $D^*(x)=\frac{1}{2}$, we should have $\alpha_1=\alpha_2$.
\end{proposition}

 \if \extended 1 A detailed proof is given in Appendix E. \fi \if \extended 0 A detailed proof is given in \cite[Appendix E]{gend2gans}. \fi Thus, we cannot have two tunable parameters in value function in the setting of $\alpha$-GANs with a single discriminator. 
Interestingly, the dual discriminator framework of our D2 $\alpha$-GANs provides us with the flexibility of using two different $\alpha$-losses in \eqref{d2-gans-value-function}.



\bibliographystyle{IEEEtran}
\bibliography{bibliofile}



\if \extended 1
\appendix


\subsection{Proof of Lemma 1}
\label{proof-lemma1}

It is easy to see that the value function in equation (3) can be re-written by expanding the expectation as
\begin{align}
    V(D_1,D_2,G) =
    \int_{s\mathcal{X}}^{ }-c_{1}\cdot  \ell_{\alpha_{1}}\left(D_{1}\left(x\right)\right)\cdot P_{d}\left(x\right)\ \nonumber
    \\ +\left(\ell_{\alpha_{2}}\left(D_{1}\left(x)\right)\right)-1\right)\cdot P_{g}\left(x\right)\
    \nonumber\\+\left(\ell_{\alpha_{2}}\left(D_{2}\left(x\right)\right)-1\right)\cdot P_{d}\left(x\right)\ \nonumber
    \\
    -c_{2}\cdot \ell_{\alpha_{1}}\left(D_{2}\left(x\right)\right)\cdot P_{g}(x)\cdot  dx \nonumber
\end{align}
 Simplifying the value function, the term inside the integral (let us call it $K(D_1,D_2,G)$) is given by :-
\begin{align}
K(D_1,D_2,G) = 
    -c_{1}\cdot\frac{\alpha_{1}}{\alpha_{1}-1}P_{d}\left(x\right)\left[1-D\left(x\right)^{\frac{\alpha_{1}-1}{\alpha_{1}}}\right] \nonumber
    \\+
    \frac{\alpha_{2}}{\alpha_{2}-1}P_{g}\left(x\right)\left[-D_{1}\left(G\left(z\right)\right)^{\frac{\alpha_{2}-1}{\alpha_{2}}}\right] \nonumber
    \\+
    \frac{\alpha_{2}}{\alpha_{2}-1}P_{d}\left(x\right)\left[-D_{2}\left(x\right)^{\frac{\alpha_{2}-1}{\alpha_{2}}}\right]\ \nonumber
    \\-
    c_{2}\cdot\frac{\alpha_{1}}{\alpha_{1}-1}P_{g}\left(x\right)\left[1-D_{2}\left(G\left(z\right)\right)^{\frac{\alpha_{1}-1}{\alpha_{1}}}\right].
    \label{K-reduced-value-function}
\end{align}
Now, let us define a function $h(t)$ as
\begin{equation*}
    h(t)=a\cdot(-\ell_{\alpha_1}(t)) +b\cdot (\ell_{\alpha_2}(t)-1).
\end{equation*}
Taking the derivative of $h(t)$ with respect to $t$ and equating to 0, we get
\begin{align}
    &h(t)=-\frac{a\cdot\alpha_1}{\alpha_1-1} \Big(1-t^{\frac{\alpha_1-1}{\alpha_1}}\Big)+b \Big(\frac{\alpha_2}{\alpha_2-1}\Big(1-t^{\frac{\alpha_2-1}{\alpha_2}}\Big)-1\Big)\nonumber \\
    &h^{'}(t^*)=a \cdot (t^*)^{-1/\alpha_1}-b\cdot (t^*)^{-1/\alpha_2}=0 \nonumber \\
    &\implies  t^*=(a/b)^{\frac{\alpha_1\alpha_2}{\alpha_2-\alpha_1}}
\end{align}
Now taking double derivative of $h(t)$ and ensuring its negativity, 
\begin{align*}
    &h^{''}(t^*)=-\frac{a}{\alpha_1}\cdot (t^*)^{-1/\alpha_1-1} + \frac{b}{\alpha_2}\cdot (t^*)^{-1/\alpha_2-1} <0 \nonumber\\
    \implies& \frac{b}{\alpha_2}(t^*)^{-1/\alpha_2-1} < \frac{a}{\alpha_1}(t^*)^{-1/\alpha_1-1}\nonumber\\
    \implies &(t^*)^{\frac{\alpha_2-\alpha_1}{\alpha_1\alpha_2}} < \frac{a}{b}\cdot \frac{\alpha_2}{\alpha_1} \nonumber \\
    \implies & \alpha_2 > \alpha_1
\end{align*}
Therefore, given $a,b,$ and $t$ are values in $\mathbb{R}^+$, $h(t)$ obtains a maxima at $t^*=(a/b)^{\frac{\alpha_1\alpha_2}{\alpha_2-\alpha_1}}$ under the condition $\alpha_2 > \alpha_1$. Now noting that the inner optimization in \eqref{d2-gans-value-function} decomposes into sum of two maximizations with the same objective function $h(t)$ over the respective discriminators $D_1$ and $D_2$, this result gives the optimal discriminators as follows. 

\begin{equation*}
        D_1^*(x) = \left( \frac{c_1 P_d(x)}{P_g(x)} \right)^{\frac{\alpha_1\alpha_2}{\alpha_2 - \alpha_1}} \ \text{and} \ D_2^*(x) = \left( \frac{c_2 P_g(x)}{P_d(x)} \right)^{\frac{\alpha_1\alpha_2}{\alpha_2 - \alpha_1}} 
\end{equation*}
\textit{for all }$x \in \mathcal{X}$.


\subsection{Proof of Theorem 1}
\label{proof-theorem1}
The value function in \eqref{d2-alpha-gans-value-function} can be written as 
\begin{align}
    &V(D_1,D_2,G) = \nonumber
    \\
    &\int_{\mathcal{X}} -c_1 \ell_{\alpha_1}(1,D_1(x)) P_d(x)
    +(\ell_{\alpha_2}(1,D_1(x))-1) P_g(x) \nonumber
    \\
    &+ (\ell_{\alpha_2}(1,D_2(x))-1) P_d(x)
    -c_2\ell_{\alpha_1}(1,D_2(x)) P_g(x) \cdot dx.
\end{align}
Let the term inside the integral be $K(D_1,D_2,G)$ so that the value function is 
\begin{equation*}
    V(D_1,D_2,G)=\int_{\mathcal{X}} K(D_1,D_2,G) \cdot dx.
\end{equation*}
 Substituting the optimal discriminators from Lemma \ref{lemma-optimal-discriminators}, we get 
\begin{align}
    K(D_1^*&,D_2^*,G) \nonumber
    \\
    =& \frac{-c_1\alpha_1}{\alpha_1-1}\cdot P_d(x)\cdot \Big(1-\Big(\frac{c_1 P_d(x)}{P_g(x)}\Big)^\frac{\alpha_1 \alpha_2 - \alpha_2}{\alpha_2 - \alpha_1} \Big) \nonumber
    \\
    & \frac{-\alpha_2}{\alpha_2-1}\cdot P_g(x) \cdot \Big(\frac{c_1 P_d(x)}{P_g(x)}\Big)^\frac{\alpha_1 \alpha_2 - \alpha_1}{\alpha_2 - \alpha_1} \nonumber
    \\
    &\frac{-\alpha_2}{\alpha_2-1}\cdot P_d(x)\cdot\Big( \frac{c_2 P_g(x)}{P_d(x)}\Big)^\frac{\alpha_1 \alpha_2 - \alpha_1}{\alpha_2 - \alpha_1} \nonumber
    \\
    &\frac{-c_2  \alpha_1}{\alpha_1-1}\cdot P_g(x) \cdot \Big(1-\Big(\frac{c_2 P_g(x)}{P_d(x)}\Big)^\frac{\alpha_1 \alpha_2 - \alpha_2}{\alpha_2 - \alpha_1}\Big) \nonumber
    \\
    =& \frac{-c_1\alpha_1}{\alpha_1-1}\cdot P_g(x)    \Big(\frac{P_d(x)}{P_g(x)}- c_1^{\frac{\alpha_1\alpha_2-\alpha_2}{\alpha_2-\alpha_1}} \Big(\frac{P_d(x)}{P_g(x)}\Big)^\frac{\alpha_1 \alpha_2 - \alpha_1}{\alpha_2 - \alpha_1} \Big) \nonumber
    \\
    & \frac{-\alpha_2}{\alpha_2-1}\cdot P_g(x) \cdot \Big(\frac{c_1 P_d(x)}{P_g(x)}\Big)^\frac{\alpha_1 \alpha_2 - \alpha_1}{\alpha_2 - \alpha_1} \nonumber
    \\
    &\frac{-\alpha_2}{\alpha_2-1}\cdot P_d(x)\cdot\Big( \frac{c_2 P_g(x)}{P_d(x)}\Big)^\frac{\alpha_1 \alpha_2 - \alpha_1}{\alpha_2 - \alpha_1} \nonumber
    \\
    &\frac{-c_2  \alpha_1}{\alpha_1-1}\cdot P_d(x)  \Big(\frac{P_g(x)}{P_d(x)}-c_2^{\frac{\alpha_1\alpha_2-\alpha_2}{\alpha_2-\alpha_1}}\Big(\frac{P_g(x)}{P_d(x)}\Big)^\frac{\alpha_1 \alpha_2 - \alpha_1}{\alpha_2 - \alpha_1}\Big).
\end{align}
Now writing the same in terms of $f^{\alpha_1,\alpha_2}_c$ defined in \eqref{d2-alpha-gans-fc}, we get
\begin{align}
  K(D_1^*, D_2^*, G)&=c_1 \cdot P_g(x)\cdot f^{\alpha_1,\alpha_2}_{c_1}\Big(\frac{P_d(x)}{P_g(x)}\Big) \nonumber
    \\&+c_2\cdot P_d(x) \cdot f^{\alpha_1,\alpha_2}_{c_2}\Big(\frac{P_g(x)}{P_d(x)}\Big).
\end{align}
The value function then becomes
\vspace{0.1cm}
\begin{align}
   &\int_{\mathcal{X}} \Bigg(c_1  P_g(x) f^{\alpha_1,\alpha_2}_{c_1}\Big(\frac{P_d(x)}{P_g(x)}\Big)+c_2 P_d(x) f^{\alpha_1,\alpha_2}_{c_2}\Big(\frac{P_g(x)}{P_d(x)}\Big) 
    \Bigg) dx \nonumber
\end{align}

\begin{align*}
    \implies V(D_1^*,D_2^*,G)&=c_1\cdot\mathbf{D}_{f^{\alpha_1, \alpha_2}_{c_1}}(P_d\|P_g) \nonumber \\
& +c_2 \cdot \mathbf{D}_{f^{\alpha_1, \alpha_2}_{c_2}}(P_g\|P_d).
\end{align*}
Minimization of this value function results to a minimum value of 
\begin{equation*}
    V(D_1^*,D_2^*,G^*)=c_1\cdot f^{\alpha_1, \alpha_2}_{c_1}(1) + c_2 \cdot f^{\alpha_1, \alpha_2}_{c_2}(1) \enspace \text{when $P_g=P_d$}.
\end{equation*}
Calculating $f_c(1)$,
\begin{align}
    &f_c(1)= \frac{-\alpha_1}{\alpha_1-1}\Big(1-c^{\frac{\alpha_1\alpha_2-\alpha_2}{\alpha_2-\alpha_1}}\Big)-\frac{\alpha_2}{\alpha_2-1}\Big(c^{\frac{\alpha_1\alpha_2-\alpha_2}{\alpha_2-\alpha_1}}\Big) \nonumber \\
    &V(D_1^*,D_2^*,G^*)=
-\frac{\alpha_1}{\alpha_1-1}(c_1+c_2) \nonumber
    \\
    &+ \Big(\frac{\alpha_1}{\alpha_1-1}-\frac{\alpha_2}{\alpha_2-1}\Big)\cdot\Big( c_1^{\frac{\alpha_1\alpha_2-\alpha_1}{\alpha_2-\alpha_1}}+c_2^{\frac{\alpha_1\alpha_2-\alpha_1}{\alpha_2-\alpha_1}} \Big).
\end{align}


\subsection{Proof of Remark 3}
\label{proof-remark3}
We prove that on applying limits $\alpha_1\rightarrow1$ and $\alpha_2\rightarrow\infty$ to \eqref{d2-alpha-gans-final-optimization}, we can recover the optimization problem of D2 GANs \cite[Equation (3)]{nguyen2017dual}. Expanding $\mathbf{D}_{f_{c_1}^{\alpha_1, \alpha_2}}(P_d\|P_g)$  using the definition of $f_c$, we get
\begin{align}
   &\mathbf{D}_{f_{c_1}^{\alpha_1, \alpha_2}}(P_d\|P_g)= \nonumber
    \\
    &\int_{\mathcal{X}} P_g(x)\cdot \Bigg( -\frac{\alpha_1}{\alpha_1-1}\Big( \frac{P_d(x)}{P_g(x)}-c_1^{\frac{\alpha_1\alpha_2-\alpha_2}{\alpha_2-\alpha_1}}\frac{P_d(x)}{P_g(x)}^{\frac{\alpha_1\alpha_2-\alpha_1}{\alpha_2-\alpha_1}}\Big) \nonumber
    \\
    &-\frac{\alpha_2}{\alpha_2-1}\Big(c_1^{\frac{\alpha_1\alpha_2-\alpha_2}{\alpha_2-\alpha_1}}\cdot \Big(\frac{P_d(x)}{P_g(x)}\Big)^{\frac{\alpha_1\alpha_2 - \alpha_1}{\alpha_2-\alpha_1}}\Big) \Bigg) \cdot dx .
\end{align}

Applying limits $\alpha_2 \rightarrow \infty$, we get

\begin{align}
    \int_{\mathcal{X}} P_g(x)\cdot \Bigg( &-\frac{\alpha_1}{\alpha_1-1}\Big( \frac{P_d(x)}{P_g(x)}-c_1^{\alpha_1-1} \Big(\frac{P_d(x)}{P_g(x)}\Big)^{\alpha_1}\Big) \nonumber
    \\
    &-1\cdot\Big(c_1^{\alpha_1-1} \Big(\frac{P_d(x)}{P_g(x)}\Big)^{\alpha_1}\Big) \Bigg) \cdot dx.
\end{align}
Now, we have 
\begin{align}
c_1 &\mathbf{D}_{f_{c_1}^{\alpha_1,\infty}}(P_d  \,\|\, P_g) 
= \nonumber
\\
&\int_{\mathcal{X}} c_1 P_g(x) \left[ -\frac{\alpha_1}{\alpha_1 - 1} \left( \frac{P_d(x)}{P_g(x)}  - c_1^{\alpha_1 - 1} \cdot \left( \frac{P_d(x)}{P_g(x)} \right)^{\alpha_1} \right)\right. \nonumber \\
&- \left. c_1^{\alpha_1-1} \left( \frac{P_d(x)}{P_g(x)} \right)^{\alpha_1} \right] dx \nonumber \\
&= \int_{\mathcal{X}} P_g(x) \Bigg[ \frac{-\alpha_1}{\alpha_1 - 1} \cdot c_1\cdot \frac{P_d(x)}{P_g(x)} \nonumber \\
&+ \frac{\alpha_1}{\alpha_1 - 1} \cdot \left( c_1 \cdot \frac{P_d(x)}{P_g(x)} \right)^{\alpha_1} - c_1^{\alpha_1} \left( \frac{P_d(x)}{P_g(x)} \right)^{\alpha_1} \Bigg] dx \nonumber \\
&= \int_{\mathcal{X}} P_g(x) \left[ \frac{-\alpha_1}{\alpha_1 - 1} \cdot c_1\cdot \frac{P_d(x)}{P_g(x)} - c_1^{\alpha_1} \left( \frac{P_d(x)}{P_g(x)} \right)^{\alpha_1} \right. \nonumber \\
&+ \left. \left(c_1 \cdot \frac{P_d(x)}{P_g(x)}\right) \cdot \frac{\alpha_1}{\alpha_1 - 1} \cdot \left( c_1 \cdot \frac{P_d(x)}{P_g(x)} \right)^{\alpha_1-1}  \right] dx \nonumber
\end{align}
\begin{align}
&= \int_{\mathcal{X}} P_g(x) \left[ \frac{-\alpha_1}{\alpha_1 - 1} \cdot c_1\cdot \frac{P_d(x)}{P_g(x)} - c_1^{\alpha_1} \left( \frac{P_d(x)}{P_g(x)} \right)^{\alpha_1} \right. \nonumber \\
&- \left(c_1 \cdot \frac{P_d(x)}{P_g(x)}\right) \cdot \frac{\alpha_1}{\alpha_1 - 1} \cdot \left(1 - \left( c_1 \cdot \frac{P_d(x)}{P_g(x)} \right)^{\alpha_1-1}\right) \nonumber \\
&+ \left. \frac{\alpha_1}{\alpha_1 - 1} \cdot c_1\cdot \frac{P_d(x)}{P_g(x)} \right] dx  
\label{eq000}
\end{align}
 Using the fact that
\begin{equation}
    \lim_{\alpha\rightarrow1} \frac{\alpha}{\alpha-1}\cdot\Big(1-t^{\alpha-1}\Big) = -\log(t),
\end{equation}
and continuing from \eqref{eq000}, we get
\begin{align}
    c_1 &\mathbf{D}_{f_{c_1}^{\alpha_1,\infty}}(P_d  \,\|\, P_g) \nonumber \\
    &= \int_{\mathcal{X}} P_g(x) \left[ c_1 \cdot \frac{P_d(x)}{P_g(x)} \cdot log\left(c_1\cdot\frac{P_d(x)}{P_g(x)}\right) - c_1\frac{P_d(x)}{P_g(x)} \right] dx \nonumber\\
    &= c_1\int_{\mathcal{X}} P_d(x) \log\left(c_1\cdot\frac{P_d(x)}{P_g(x)}\right) dx - c_1\int_{\mathcal{X}} P_d dx \nonumber\\
    &= c_1\int_{\mathcal{X}} P_d(x) \log(c_1) \enspace dx + c_1\int_{\mathcal{X}} P_d(x) \log\left(\frac{P_d(x)}{P_g(x)}\right)\enspace dx \nonumber\\
    &- c_1\int_{\mathcal{X}} P_d(x) \enspace dx \nonumber.\\
    &\text{Rearranging the terms, we get} \nonumber\\
    &c_1 \mathbf{D}_{f_{c_1}^{\alpha_1, \alpha_2}}(P_d \,\|\, P_g)= c_1 \log c_1 - c_1 + c_1 \mathbf{D}_{\text{KL}}(P_d \,\|\, P_g).
\end{align}
Similarly, for taking limits on  $c_2\cdot\mathbf{D}_{f_{c_2}^{\alpha_1, \alpha_2}}(P_g\|P_d)$, we can change $c_1$ to $c_2$ and interchange $P_g$ and $P_d$ in the above analysis This recovers the D2 GANs value function \cite[Theorem 2]{nguyen2017dual}.



\subsection{Proof of Theorem 2}
\label{proof-theorem2}
We need to prove that for any functions $\ell_1$ and $\ell_2$ with inputs from $\mathbb{R}^+$, the optimization problem of the value function in \eqref{general-loss-value-function} simplifies to minimization of a linear combination of $f$ and \textit{reverse} $f$-divergences. Note that we can separate the $\sup_{D_1}$ individually over the first 2 terms and $\sup{D_2}$ over the next 2 terms because the first 2 terms contain only $D_1$ and the next 2 terms contain only $D_2$. Focusing on the first 2 terms of the value function and considering the $\sup_{D_1}$ optimization over them, we get 
\begin{align}
    \sup_{D_1 : x \rightarrow\mathbb{R}^+} \sum_x &  -c_1 P_d(x) \ell_1(D_1(x)) + P_g(x) \ell_2(D_1(x))
    \nonumber \\
    =& \sum_x \sup_{t \in \mathbb{R}^+}  -c_1 P_d(x) \ell_{1}(t) + P_g(x) \ell_{2}(t) \nonumber\\
    =& \sum_x P_g(x) \cdot c_1 \sup_{t \in \mathbb{R}^+} \Big( -\frac{P_d(x)}{P_g(x)} \ell_{1}(t) + \frac{1}{c_1}\ell_2(t)\Big)\nonumber\\    
    =& c_1 \sum_x P_g(x) f_{c_1}(\frac{P_d(x)}{P_g(x)}) \nonumber \\
    =& c_1 \mathbb{D}_{f_{c_1}}(P_d \| P_g),  \nonumber \\
    \text{where}\quad & f_{c_1}(u) = \sup_{t \in \mathbb{R}^+} \Big[-u \cdot \ell_{1}(t) + \frac{1}{c_1}\ell_{2}(t)\Big].
\end{align}

 Also, $f(u)$ is a convex function since supremum of a set of affine functions is always convex \cite{nguyen2009surrogate}. Therefore, the final term is nothing but an $f$-divergence. Similarly, the next 2 terms in the value function evaluate to another $f$-divergence.

\begin{align}
    \sup_{D_2 : x \rightarrow\mathbb{R}^+} &\sum_x  P_d(x) \ell_{2}(D_2(x))  
    -c_2 P_g(x) \ell_{1}(D_2(x)) \nonumber\\
    &= \sum_x \sup_{t \in \mathbb{R}^+} 
    P_d(x) \ell_{2}(t)  
    -c_2 P_g(x) \ell_{1}(t)
    \nonumber\\
    &= \sum_x c_2 \cdot P_g(x) \sup_{t \in \mathbb{R}^+} \Big(\frac{P_d(x)}{P_g(x)} \frac{\ell_{2}(t)}{c_2} -  \ell_1(t)\Big) \nonumber \\   
    & = c_2 \sum_x P_g(x) f^{'}_{c_2}\Big(\frac{P_d(x)}{P_g(x)}\Big) \nonumber\\
    &= c_2 \mathbb{D}_{f^{'}_{c_2}}(P_d \| P_g) \nonumber \\
    &= c_2 \mathbb{D}_{f_{c_2}}(P_g \| P_d), \nonumber \\
    \text{where} \quad &  f^{'}_{c_2}(u) = u\cdot f_{c_2}(1/u) = \sup_{t \in \mathbb{R}^+} \Big[ u \frac{\ell_{2}(t)}{c_2} - \ell_{1}(t)\Big].
\end{align}

\subsection{Proof of Proposition 1}
\label{proof-proposition1}
Consider
\begin{align*}
    V(D, G) &= \mathbb{E}_{X \sim P_d}[-\ell_{\alpha_1}(1,D(X))] \nonumber\\
    &+ \mathbb{E}_{X \sim P_{g}}[-\ell_{\alpha_2}(0,D(X))]\nonumber
\end{align*}
\begin{align}
    &=\int_{\mathcal{X}}-P_d(x)\cdot\ell_{\alpha_1}(1,D(x)) - P_{g}(x)\cdot(\ell_{\alpha_2}(0,D(x))) \enspace dx\nonumber\\
    &= -\int_{\mathcal{X}} P_d(x)\cdot\frac{\alpha_1}{\alpha_1-1}\cdot\Big(1-D(x)^{\frac{\alpha_1-1}{\alpha_1}}\Big)\nonumber\\
    &\quad\text{ } -P_{g}(x)\cdot \frac{\alpha_2}{\alpha_2-1}\cdot\Big(1-(1-D(x))^{\frac{\alpha_2-1}{\alpha_2}}\Big) \enspace dx
\end{align}
Differentiating the expression inside integral with respect to discriminator $D(x)$ and equating it to 0 to find the optimal discriminators, we get
\begin{align}
    &P_d(x)\cdot D(x)^{-1/\alpha_1}-P_{g}(x)\cdot (1-D(x))^{-1/\alpha_2}=0 \nonumber
    \\
    &\implies P_d(x)D(x)^{-1/\alpha_1}=P_g(x)(1-D(x))^{-1/\alpha_2}
\end{align}
At convergence, when $P_d=P_{g}$, optimal discriminator should output real and fake labels with equal probability. Thus, $D(x)=\frac{1}{2}$ should hold. However, from the above equations, we get 
\begin{equation*}
    D(x)^{\alpha_2}=(1-D(x))^{\alpha_1}
\end{equation*}
If the optimal discriminator is given by $D^*(x)=\frac{1}{2}$, for all $x\in\mathcal{X}$, then the above implies that $\alpha_1=\alpha_2$, which is a contradiction. Therefore, the use of two different $\alpha$-losses in \eqref{General-Loss-GANs} is not meaningful.

\fi

\end{document}